\documentclass[twoside]{article}

\usepackage{arxiv}

\usepackage[utf8]{inputenc} 
\usepackage[T1]{fontenc}    
\usepackage{url}            
\makeatletter
\g@addto@macro{\UrlBreaks}{\UrlOrds}
\makeatother
\usepackage{booktabs}       
\usepackage{nicefrac}       
\usepackage{microtype}      
\usepackage{multirow}

\usepackage{bm}
\usepackage{amsthm}
\usepackage{amsfonts}       
\usepackage{amsmath}
\usepackage{graphicx}
\usepackage{natbib}
\usepackage{amssymb}%

\usepackage{makecell}
\usepackage{tablefootnote}
\usepackage{mathrsfs}%
\usepackage[title]{appendix}%
\usepackage{textcomp}%
\usepackage{manyfoot}%
\usepackage{booktabs}%
\usepackage{algorithm}%
\usepackage{algorithmicx}%
\usepackage{algpseudocode}%
\usepackage{listings}%

\usepackage{hyperref}       

\newtheorem{thm}{Proposition}
\newtheorem{dfn}{Definition}

\newcommand{\I}{\bm{I}}
\newcommand{\R}{\mathbb{R}}

\newcommand{\K}{\bm{K}}

\newcommand{\U}{\bm{U}}

\newcommand{\fhv}{\bm{\hat{f}}}
\newcommand{\fh}{\hat{f}}
\newcommand{\xv}{\bm{x}}
\newcommand{\xvi}{\bm{x_i}}
\newcommand{\xvj}{\bm{x_j}}
\newcommand{\xpvj}{\bm{x'_j}}
\newcommand{\xsv}{\bm{x^*}}

\newcommand{\xpv}{\bm{x'}}

\newcommand{\betav}{\bm{\beta}}

\newcommand{\yv}{\bm{y}}

\newcommand{\dv}{\bm{d}}
\newcommand{\dvi}{\bm{d_i}}
\newcommand{\kv}{\bm{k}}
\newcommand{\X}{\bm{X}}

\newcommand{\Xp}{\bm{X'}}
\newcommand{\Hh}{\mathcal{H}}

\title{Bandwidth Selection for Gaussian Kernel Ridge Regression via Jacobian Control}
\author{%
  Oskar Allerbo \\
  Mathematical Sciences\\
  University of Gothenburg and Chalmers University of Technology\\
  \texttt{allerbo@chalmers.se} \\
  \And
  Rebecka J\"ornsten\\
  Mathematical Sciences\\
  University of Gothenburg and Chalmers University of Technology\\
  \texttt{jornsten@chalmers.se} \\
}

\begin{document}
\maketitle

\begin{abstract}
Most machine learning methods require tuning of hyper-parameters. For kernel ridge regression with the Gaussian kernel, the hyper-parameter is the bandwidth. The bandwidth specifies the length scale of the kernel and has to be carefully selected to obtain a model with good generalization. The default methods for bandwidth selection, cross-validation and marginal likelihood maximization, often yield good results, albeit at high computational costs. 
Inspired by Jacobian regularization, we formulate an approximate expression for how the derivatives of the functions inferred by kernel ridge regression with the Gaussian kernel depend on the kernel bandwidth. We use this expression to propose a closed-form, computationally feather-light, bandwidth selection heuristic, based on controlling the Jacobian. In addition, the Jacobian expression illuminates how the bandwidth selection is a trade-off between the smoothness of the inferred function and the conditioning of the training data kernel matrix. We show on real and synthetic data that compared to cross-validation and marginal likelihood maximization, our method is on pair in terms of model performance, but up to six orders of magnitude faster.
\end{abstract}


\section{Introduction}
Kernel ridge regression, KRR, is a non-linear, closed-form solution regression technique 
used within a wide range of applications \citep{zahrt2019prediction, ali2020complete, chen2021optimizing, fan2021well, le2021fingerprinting, safari2021kernel, shahsavar2021experimental, singh2021neural, wu2021increasing, chen2022kernel}. It is related to Gaussian process regression \citep{krige1951statistical, matheron1963principles, williams2006gaussian}, but with a frequentist, rather than a Bayesian, perspective.
Apart from being useful on its own merits, in recent years, the similarities between KRR and neural networks have been highlighted, making the former an increasingly popular tool for gaining better theoretical understandings of the latter \citep{belkin2018understand, jacot2018neural, chen2020deep, geifman2020similarity, ghorbani2020neural, ghorbani2021linearized, mei2021generalization}.

However, kernelization introduces hyper-parameters, which need to be carefully tuned in order to obtain good generalization. The bandwidth, $\sigma$, is a hyper-parameter used by many kernels, including the Gaussian, or radial basis function, kernel. The bandwidth specifies the length scale of the kernel.
A kernel with a too small bandwidth will treat most new data as far from any training observation, while a kernel with a too large bandwidth will treat each new data point as basically equidistant to all training observations. None of these situations will result in good generalization.

The problem of bandwidth selection has been extensively studied for kernel density estimation, KDE, which is the basis for KDE-based kernel regression, such as the Nadaraya-Watson estimator \citep{nadaraya1964estimating, watson1964smooth} and locally weighted regression \citep{cleveland1988locally}. \citet{kohler2014review} review existing methods for bandwidth selection for KDE-based kernel regression, methods that all make 
varyingly strong assumptions on the underlying data density and smoothness of the non-parametric regression model, and how the latter can be approximately estimated. On one end of the spectrum, cross-validation and marginal likelihood maximization make almost no assumptions on the underlying data structure, resulting in very flexible, but computationally heavy, estimators, with high variance. On the other end, with strong assumptions on the underlying data structure, Silverman's rule of thumb, originally from 1986, \citep{silverman2018density} is a computationally light estimator with low variance, but possibly far from optimal. Other approaches on the spectrum include \citet{park1990comparison}, \citet{sheather1991reliable}, and \citet{fan1995data}.

Although similar in name and usage, KRR and KDE-based kernel regression are not the same. While KDE-based kernel regression estimates the probability density of the data, and uses this density to estimate $\mathbb{E}(y|\xv)$, KRR takes a functional perspective, directly estimating $\hat{y}=\fh(\xv)$, similarly to how is done in neural networks.  

For neural networks, Jacobian regularization, which penalizes the Frobenius norm of the Jacobian, $\left\|\frac{\partial \fhv(\xv)}{\partial \xv}\right\|_F^2$, has recently been successfully applied to improve generalization \citep{jakubovitz2018improving, chan2019jacobian, hoffman2019robust, finlay2020train, bai2021stabilizing}.
The Jacobian penalty is a non-linear generalization of the linear ridge penalty. To see this, consider the linear model $\fh(\xv)=\xv^\top\betav$, for which the Jacobian penalty becomes exactly the ridge penalty, $\|\betav\|_2^2$. Thus, both Jacobian and ridge regularization improve generalization by constraining the derivatives of the inferred function. 

This connection motivates our investigation into how Jacobian constraints can be applied for bandwidth selection in KRR: If we knew how the kernel bandwidth affects the Jacobian of the inferred function, then we could use Jacobian control as a criterion for selecting the bandwidth.

Our main contributions are:
\begin{itemize}
\item
We derive an approximate expression for the Jacobian of the function inferred by KRR with the Gaussian kernel.
\item
We propose a closed-form, computationally feather-light, bandwidth selection method for KRR with the Gaussian kernel, based on controlling the approximate Jacobian. 
\item
We show on synthetic and real data that Jacobian-based bandwidth selection outperforms cross-validation and marginal likelihood maximization in terms of computation speed, and Silverman's method in terms of model performance.
\end{itemize}

\section{Bandwidth Selection through Jacobian Control}
\label{sec:method}
Consider the left panel of Figure \ref{fig:bandwidth_demo}. When large (absolute) derivatives of the inferred function are allowed, the function varies more rapidly between observations, while a function with constrained derivatives varies more smoothly, which intuitively improves generalization. However, the derivatives must not be too small as this leads to an overly smooth estimate, as seen in the right panel.

\begin{figure}
\center
\includegraphics[width=\textwidth]{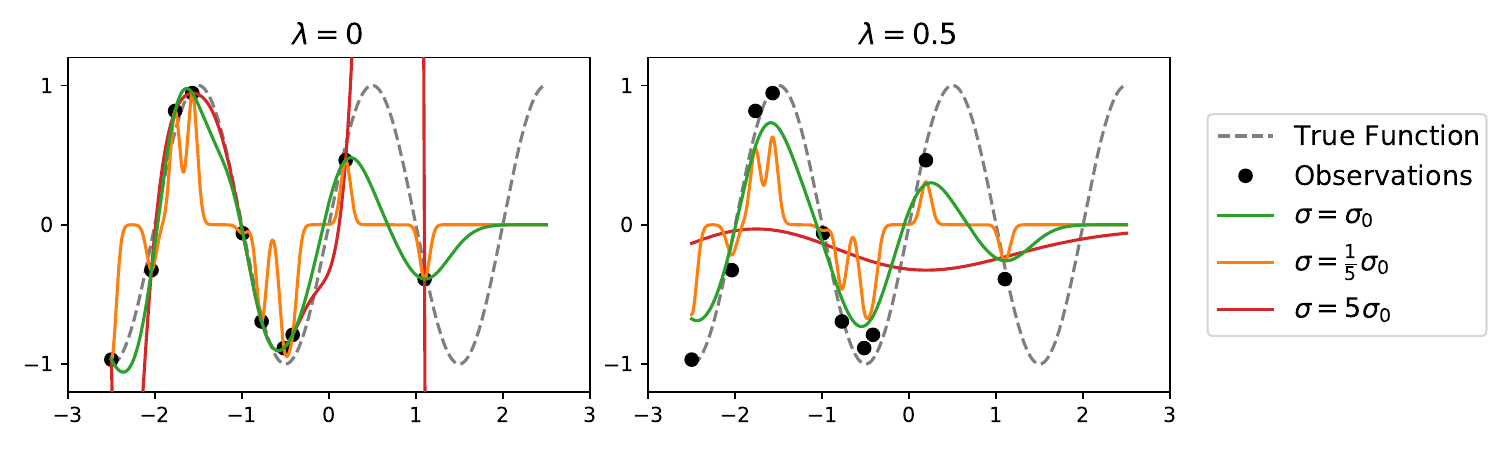}
\caption{Kernel ridge regression with different bandwidths and different regularizations, where $\sigma_0$ is the bandwidth proposed by the Jacobian method, and $\lambda$ is the strength of the regularization. In the absence of regularization, regardless of the bandwidth, the inferred function perfectly interpolates the training data, i.e.\ it hits all training observations. When the bandwidth is too small, the kernel considers most new observations as far away from any training data and quickly resorts to its default value, 0. A too large bandwidth, on the other hand, results in extreme predictions between some of the observations. The addition of regularization affects larger bandwidths more than smaller ones. A too large bandwidth, in combination with regularization, produces a function that is too simple to capture the patterns in the training data.}
\label{fig:bandwidth_demo}
\end{figure}

The functions in Figure \ref{fig:bandwidth_demo} are all constructed using kernel ridge regression, KRR, with the Gaussian kernel. For training data $\X\in\R^{n\times p}$ and $\yv \in \R^n$, the objective function of KRR is
\begin{equation}
\label{eq:krr}
\min_{f\in \Hh_k}\left\|\yv-\begin{bmatrix} f(\bm{x_1})&\dots&f(\bm{x_n})\end{bmatrix}^\top\right\|_2^2+\lambda\left\|f\right\|_{\Hh_k}^2.
\end{equation}
$\Hh_k$ denotes the reproducing kernel Hilbert space corresponding to the symmetric, positive semi-definite kernel function $k(\xv,\xpv)$, and $\lambda\geq 0$ is the regularization strength. Solving Equation \ref{eq:krr}, a prediction $\fh(\xsv)\in\R$, where $\xsv\in\R^p$, is given by
\begin{equation}
\label{eq:krr_s}
\fh(\xsv)=\kv\left(\xsv,\X\right)^\top\cdot\left(\K\left(\X,\X\right)+\lambda\I\right)^{-1}\cdot\yv,
\end{equation}
where $\kv(\xsv,\X)\in \R^{n}$ and $\K(\X,\X)\in \R^{n\times n}$ are two kernel matrices, $\kv(\xsv,\X)_{i}=k(\xsv,\xvi)$ and $\K(\X,\Xp)_{ij}=k(\xvi,\xpvj)$.

The Gaussian kernel is given by
\begin{equation}
\label{eq:gauss_kern}
k_G(\xv,\xpv,\sigma):=\exp\left(-\frac{\|\xv-\xpv\|_2^2}{2\sigma^2}\right),
\end{equation}
where the bandwidth, $\sigma$, specifies the length-scale of the kernel, i.e.\ what is to be considered as ``close''. 

Returning to our aspiration from above, we would like to select $\sigma$ in Equation \ref{eq:gauss_kern} to control $\left\|\frac{\partial \fh(\xsv)}{\partial \xsv}\right\|_F=\left\|\frac{\partial \fh(\xsv)}{\partial \xsv}\right\|_2$, with $\fh(\xsv)$ given by Equation \ref{eq:krr_s}. 

In general, there is no simple expression for $\left\|\frac{\partial \fh(\xsv)}{\partial \xsv}\right\|_2$, but in 
Definition \ref{dfn:approx_jac}, we state an approximation that is based on derivations that we will present in Section \ref{sec:jac_just}.

\begin{dfn}[Approximate Jacobian Norm]
\label{dfn:approx_jac}
\begin{equation}
\label{eq:approx_jac}
\begin{aligned}
J_2^a(\sigma)&=J_2^a(\sigma,l_{\max},n,p,\lambda):=\\
&\frac1{\sigma}\cdot\frac1{n\cdot\exp\left(-\left(\frac{((n-1)^{1/p}-1)\pi\sigma}{2l_{\max}}\right)^2\right)+\lambda}
\cdot C(n,\|\yv\|_2),
\end{aligned}
\end{equation}
\end{dfn}
where $l_{\max}$ denotes the maximum distance between two training observations, and $C(n,\|\yv\|_2)$ is a constant with respect to $\sigma$. 

\textbf{Remark 1}: Since we are only interested in how $J^a_2$ depends on $\sigma$, we will henceforth, with a slight abuse of notation, omit the constant $C(n,\|\yv\|_2)$. 

\textbf{Remark 2}: Technically, since we use a univariate response, $\frac{\partial \fh(\xsv)}{\partial \xsv}$ is a special case of the Jacobian, the gradient. We chose, however, to use the word Jacobian, since nothing in our derivations restricts us to the univariate case.

\textbf{Remark 3}: We will refer to the two $\sigma$ dependent factors in Equation \ref{eq:approx_jac} as 
\begin{equation}
\label{eq:jajb}
\begin{aligned}
j_a(\sigma):=\frac1\sigma\quad\text{and}\quad
j_b(\sigma):=\frac1{n\cdot\exp\left(-\left(\frac{((n-1)^{1/p}-1)\pi\sigma}{2l_{\max}}\right)^2\right)+\lambda}.
\end{aligned}
\end{equation}

Proposition \ref{thm:best_sigma} below characterizes how the approximate Jacobian norm, $J^a_2$, depends on $\sigma$. Depending on $\lambda$, it can behave in three different ways: In the absence of regularization, $J^a_2$ becomes arbitrarily large for $\sigma$ small or large enough and enjoys a global minimum, $\sigma_0$, which is consistent with the left panel of Figure \ref{fig:bandwidth_demo}. However, as soon as regularization is added, $J^a_2$ goes to 0 as bandwidth goes to infinity, as indicated in the right panel. As long as the regularization parameter $\lambda\leq 2ne^{-3/2}\approx 0.45n$, there still exists a local minimum at $\sigma_0$. This is further illustrated in Figure \ref{fig:appr_jac_demo}, where we plot $J^a_2$ together with its components $j_a$ and $j_b$ for three different values of $\lambda$, reflecting the three types of behavior. 
Since $j_b$ is bounded by $1/\lambda$, for $\lambda$ large enough it is negligible compared to $j_a$.

\begin{figure}
\center
\includegraphics[width=\textwidth]{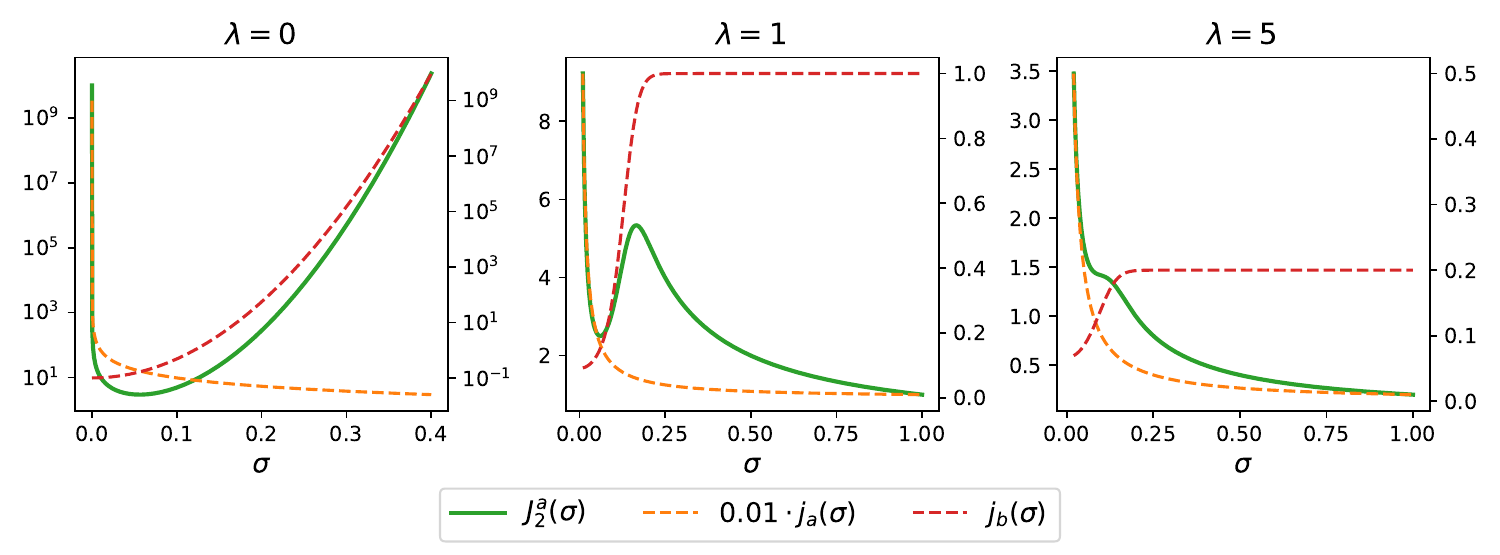}
\caption{Approximate Jacobian norm, $J^a_2$, (left y-axis), and its two factors, $j_a$ and $j_b$, (right y-axis), as defined in Equations \ref{eq:approx_jac} and \ref{eq:jajb}, as functions of the bandwidth for three different values of $\lambda$, and $n=10$, $l_{\max}=1$, $p=1$. Note that the scales of the axes differ between the three panels and that $j_a(\sigma)$ is scaled down by a factor of 100. We clearly see the global and local minima stated in Proposition \ref{thm:best_sigma}. For $\lambda=0$, $J^a_2(0)=J^a_2(\infty)=+\infty$ with a global minimum at $\sigma_0$. For $\lambda>0$, $J^a_2(0)=+\infty$ and $J^a_2(\infty)=0$. For $\lambda\leq2ne^{-3/2}$, $J^a_2$ has a local minimum at $\sigma_0$. $j_a(\sigma)$ decreases monotonically to 0, while $j_b(\sigma)$ increases monotonically to $1/\lambda$.}
\label{fig:appr_jac_demo}
\end{figure}

\begin{thm}~\\
\label{thm:best_sigma}
Let $J^a_2(\sigma)$ be defined according to Definition \ref{dfn:approx_jac}, and let, for $k\in\{-1,0\}$, 
\begin{equation}
\label{eq:best_sigma}
\sigma_k:=\frac{\sqrt{2}}{\pi}\frac{l_{\max}}{(n-1)^{1/p}-1} \sqrt{1-2W_k\left(-\frac{\lambda \sqrt{e}} {2n}\right)},
\end{equation}
where $W_k$ denotes the $k$-th branch of the Lambert $W$ function.
Then
\begin{itemize}
\item
For $\lambda=0$, $J^a_2(0)=J^a_2(\infty)=+\infty$, and $J^a_2(\sigma_0)=J^a_2\left(\frac{\sqrt{2}}{\pi}\frac{l_{\max}}{(n-1)^{1/p}-1}\right)$ is a global minimum.
\item
For $0<\lambda\leq2ne^{-3/2}$, $J^a_2(0)=+\infty$, and $J^a_2(\infty)=0$, with a local minimum $J^a_2(\sigma_0)$ and a local maximum $J^a_2(\sigma_{-1})$.
\item
For $\lambda>2ne^{-3/2}$, neither $\sigma_0$ nor $\sigma_{-1}$ is defined and $J^a_2(\sigma)$ decreases monotonically from $J^a_2(0)=+\infty$ to $J^a_2(\infty)=0$.
\end{itemize}
\end{thm}

\begin{proof}[Sketch of Proof]
We evaluate $J^a_2(0)$, $J^a_2(\infty)$ and points where $\frac{\partial J^a_2(\sigma)}{\partial \sigma}=0$ to obtain extreme point candidates. To avoid infeasible solutions, we have to consider the domain of the Lambert W function. For the full proof, see Appendix \ref{sec:proofs}.
\end{proof}

Based on Proposition \ref{thm:best_sigma} we can now propose a bandwidth selection scheme based on Jacobian control: 
For $\lambda\leq2ne^{-3/2}$, we choose the (possibly local) minimum $\sigma_0$ as our Jacobian based bandwidth. For $\lambda>2ne^{-3/2}$, $\sigma_0$ is not defined; in this case we choose our bandwidth as if $\lambda=2ne^{-3/2}$. Note that $\sigma_0$ is quite stable to changes in $\lambda$: The square root expression in Equation \ref{eq:best_sigma} increases from $1$ for $\lambda=0$ to $\sqrt{3}$ for $\lambda=2ne^{-3/2}$. This stability in terms of $\lambda$ can be seen in Figures \ref{fig:appr_jac_demo} and \ref{fig:bw_sim_demo}.

\subsection{Theoretical Details}
\label{sec:jac_just}
In this section, we present the calculations behind Definition \ref{dfn:approx_jac}. We also illuminate how bandwidth selection is a trade-off between a well-conditioned training data kernel matrix and a slow decay of the inferred function toward the default value.

We first use Proposition \ref{thm:dydx} to approximate the norm of the two kernel matrices in Equation \ref{eq:krr_s} with a product of two matrix norms. We then use Propositions \ref{thm:max_der_gauss} and \ref{thm:min_ki_gauss} to estimate these two norms for the case of the Gaussian kernel. Note that Proposition \ref{thm:dydx} holds for any kernel, not only the Gaussian.
\begin{thm}~\\
\label{thm:dydx}
Let $\dvi:=\xsv-\xvi$ where $\xvi$ is a row in $\X$.
Then, with $\fh(\xsv)$ according to Equation \ref{eq:krr_s}, for any function $k(\xv,\xpv)$,
\begin{equation}
\label{eq:dydx}
\begin{aligned}
\left\|\frac{\partial \fh(\xsv)}{\partial \xsv}\right\|_2&=\left\|\frac{\partial \fh(\xsv)}{\partial \dvi}\right\|_2\\
&\leq\max_{\xvi\in\X}\left\|\frac{\partial k(\xsv,\xvi)}{\partial \xsv}\right\|_1
\cdot\left\|\left(\K(\X,\X)+\lambda\I\right)^{-1}\right\|_2\cdot\sqrt{n}\left\|\yv\right\|_2\\
&=\max_{\xvi\in\X}\left\|\frac{\partial k(\dvi+\xvi,\xvi)}{\partial \dvi}\right\|_1\cdot\left\|\left(\K(\X,\X)+\lambda\I\right)^{-1}\right\|_2\cdot\sqrt{n}\left\|\yv\right\|_2,\\
\end{aligned}
\end{equation}
where the matrix norms are the induced operator norms.
\end{thm}
\begin{proof}[Sketch of Proof]
We first use submultiplicativity to split the norm into three factors, and equivalence of matrix norms to replace the 2-norm with the 1-norm in the first factor. Finally, we show that $\frac{\partial \fh(\xsv)}{\partial \xsv}=\frac{\partial \fh(\xsv)}{\partial \dvi}$ to obtain the forms for the two factors. For the full proof, see Appendix \ref{sec:proofs}.
\end{proof}

\begin{thm}~\\
\label{thm:max_der_gauss}
Let $\dvi:=\xsv-\xvi$ where $\xvi$ is a row in $\X$, and denote $d_i:=\|\dvi\|_2$. Then, for the Gaussian kernel,\\ $k_G(\dvi,\sigma)=\exp\left(-\frac{\|\dvi\|_2^2}{2\sigma^2}\right)$,
\begin{equation}
\label{eq:max_der_gauss}
\begin{aligned}
\max_{\xvi\in\X}\left\|\frac{\partial k_G(\dvi,\sigma)}{\partial \dvi}\right\|_1=\max_{\xvi\in\X}\frac{d_i}{\sigma^2}\exp\left(-\frac{d_i^2}{2\sigma^2}\right)
\leq\frac1{\sigma\sqrt{e}}=:\frac{1}{\sqrt{e}}\cdot j_a(\sigma).
\end{aligned}
\end{equation}
\end{thm}
\begin{proof}[Sketch of Proof]
Since the Gaussian kernel is rotationally invariant we only need to consider the radial coordinate, $d_i$, when calculating the gradient. The value of $d_i$ that maximizes the gradient is calculated by setting the derivative (of the gradient) to zero. For the full proof, see Appendix \ref{sec:proofs}.
\end{proof}

\begin{thm}~\\
\label{thm:min_ki_gauss}
For $\K(\X,\X,\sigma)_{ij}=k_G(\xvi,\xvj,\sigma)$, where $k_G(\xv,\xpv,\sigma)$ denotes the Gaussian kernel,
\begin{equation}
\label{eq:min_ki_gauss}
\begin{aligned}
\left\|\left(\K(\X,\X,\sigma)+\lambda\I\right)^{-1}\right\|_2
\geq \frac1{n\cdot\exp\left(-\left(\frac{((n-1)^{1/p}-1)\pi\sigma}{2l_{\max}}\right)^2\right)+\lambda}=:j_b(\sigma).
\end{aligned}
\end{equation}
\end{thm}
\begin{proof}[Sketch of Proof]
\citet{bermanis2013multiscale} provide an estimate for the number of singular values larger than $\delta\cdot s_1$ for a Gaussian kernel matrix, where $\delta>0$ and $s_1$ denotes the largest singular value. Using this expression, we can upper bound the smallest singular value, or, equivalently, lower bound the largest singular value of the inverse matrix. The addition of $\lambda\I$ shifts all singular values by $\lambda$. For the full proof, see Appendix \ref{sec:proofs}.
\end{proof}

In the absence of regularization, $j_b(\sigma)$ is a bound of the spectral norm of the inverse training data kernel matrix. With increasing $\sigma$, the elements in $\K(\X,\X,\sigma)$ become increasingly similar, and $\K(\X,\X,\sigma)$ becomes closer to singular, which results in an ill-conditioned solution, where $\fh(\xsv)$ is very sensitive to perturbations in $\X$. Introducing regularization controls the conditioning, as seen in Figure \ref{fig:appr_jac_demo}; $j_b(\sigma)$ is upper bounded by $1/\lambda$. The poor generalization properties of regression with an ill-conditioned kernel matrix are well known, see e.g. \citet{poggio2019double}, \citet{amini2021spectrally}, or \citet{hastie2022surprises}.

With the Jacobian approach, the contribution of an ill-conditioned matrix, $j_b(\sigma)$, is balanced by how quickly the inferred function decays in the absence of training data, $j_a(\sigma)$. For a too small bandwidth, the inferred function quickly decays to zero; $j_a(\sigma)$ is large and thus the derivatives of the inferred function. For a too large bandwidth, $\K(\X,\X,\sigma)$ is almost singular, which results in extreme predictions; $j_b(\sigma)$ is large, and thus the derivatives of the inferred function. By controlling the Jacobian, both poor generalization due to predicting mostly zero and poor generalization due to extreme predictions is avoided.

\subsection{Outlier Sensitivity}
\label{sec:outliers}
Since $l_{\max}$ might be sensitive to outliers, Equation \ref{eq:best_sigma} suggests that so might the Jacobian method. One option to mitigate this problem is to use a trimmed version of $l_{\max}$, calculated after removing outliers. Our approach is however based on the observation that for data evenly spread within a hypercube in $\R^p$ with side $l_{\max}$, $\frac{l_{\max}}{(n-1)^{1/p}-1}$ is exactly the distance from an observation to its closest neighbor(s). We thus define the Jacobian median method analogously to the Jacobian method but with $\frac{l_{\max}}{(n-1)^{1/p}-1}$ replaced by $\text{Med}_{i=1}^n\left(\min_{j\neq i}\|\xv_i-\xv_j\|_2\right)$, i.e. median of the closest-neighbor-distances.

\section{Experiments}
\label{sec:experiments}
Experiments were performed on one synthetic and seven real data sets. The metrics evaluated were test $R^2$, i.e.\ the proportion of the variation in the data that is explained by the model, the selected bandwidth $\sigma$, and bandwidth selection computation time in milliseconds, $t$. 
In all experiments, the Jacobian-based bandwidth was compared to those of generalized cross-validation (GCV) for kernel regression \citep{hastie2022surprises}, marginal likelihood maximization (MML), and Silverman's rule of thumb. 
Note, however, that Silverman's method was developed with KDE, rather than KRR, in mind and thus does not take $\lambda$ into account. We chose however to include it as a reference since, just like the Jacobian method, it is a computationally light, closed-form solution. To avoid singular matrices, a small regularization of $\lambda=10^{-3}$ was used in all experiments unless otherwise stated.
All experiments were run on a cluster with Intel Xeon Gold 6130 CPUs.

In Section \ref{sec:large_exps}, we perform experiments on five large data sets, focusing on performance, in terms of $R^2$ and computation time, while in Section \ref{sec:simple_exps}, we perform experiments on relatively simple data sets, to demonstrate the similarities and differences between the four bandwidth selection methods in greater detail.

\subsection{Large Real Data}
\label{sec:large_exps}
In this section, we compare the bandwidth selection methods on the five real data sets described in Table \ref{tab:big_data}. The data sets were selected to compare the algorithms on a diverse set of applications, although they all have in common that they are relatively large in terms of number of observations and dimensions.

\begin{table}
\caption{Real data sets used for comparing the four bandwidth selection methods.}
\centering
\begin{tabular}{|l|l|}
\hline
Data Set & Size, $n\times p$\\
\hline
\hline
Quality of aspen tree fibres\tablefootnote{The data set is available at \url{https://openmv.net/info/wood-fibres}.} & $25165\times 5$\\
\hline
\makecell[l]{Appliances energy use in a low energy building in\\Stambruges, Belgium \citep{candanedo2017data}\tablefootnote{The data set is available at \url{https://github.com/LuisM78/Appliances-energy-prediction-data}.}} & $19735\times 27$\\
\hline
House values in California \citep{pace1997sparse}\tablefootnote{The data set is available at \url{https://www.dcc.fc.up.pt/~ltorgo/Regression/cal_housing.html}.} & $20640\times  8$\\
\hline
\makecell[l]{Protein structure as root-mean-square deviation\\ of atomic positions, taken from CASP\tablefootnote{\url{https://predictioncenter.org/}, the data set is available at \url{https://archive.ics.uci.edu/ml/datasets/Physicochemical+Properties+of+Protein+Tertiary+Structure}.}} & $45730\times 9$\\
\hline
\makecell[l]{Daily concentration of black smoke particles in the U.K.\\ in the year 2000 \citep{wood2017generalized}\tablefootnote{The data set is available at \url{https://www.maths.ed.ac.uk/~swood34}.}} & $45568\times 10$\\
\hline
\end{tabular}
\label{tab:big_data}
\end{table}

For each data set, 100 random splits were created by selecting 10000 of the observations at random. For each split, the data was standardized and split randomly into 65 \% training and 45 \% testing data. For cross-validation, 10 logarithmically spaced values between $0.001$ and $l_{\max}$ were used. 

In Table \ref{tab:big_res}, we present the means together with the first and ninth deciles for $R^2$, and $t$ across the 100 splits for each method, together with the results of Wilcoxon signed rank tests, testing whether the Jacobian method performs better (in terms of explained variance) and faster (in terms of computation time) than the competing methods. On the evaluated data, the Jacobian method is up to a million times faster than GCV and MML, while generally performing on par, or better, in terms of $R^2$. 
Compared to MML, the Jacobian method performs significantly better on four of the five data sets, while GCV tends to perform slightly better than the Jacobian method. Silverman's method also performs well in terms of computation time but performs significantly worse than the Jacobian method in terms of $R^2$.
\begin{table}
\centering
\caption{Mean together with first and ninth deciles (within parentheses) of explained variance, $R^2$, and bandwidth selection time in seconds, $t$, for the five data sets from Table \ref{tab:big_data}. In the second rows of the cells, we state the p-values of Wilcoxon signed rank tests, testing whether the Jacobian method performs better (in terms of explained variance) and/or faster (in terms of computation time) than the competing method.
In all cases, the Jacobian method is significantly faster than the competing methods, and, in most cases, it performs significantly better.}
\begin{tabular}{|l|l|l|l|}
\hline
\hline
 \multirow{7}{*}{\makecell{Aspen\\Fibres}} & Jacobian  & $0.65$ ($0.62$, $0.67$) & $0.00031$ ($0.00016$, $0.00043$)\\
\cline{2-4}
 & \multirow{2}{*}{GCV}  & $0.65$ ($0.63$, $0.67$) & $520$ ($350$, $690$)\\
 & & $p_{\text{Wil}}=0.59$& $p_{\text{Wil}}=\bm{1.9\cdot 10^{-18}}$\\
\cline{2-4}
 & \multirow{2}{*}{MML}  & $0.61$ ($0.59$, $0.64$) & $280$ ($240$, $330$)\\
 & & $p_{\text{Wil}}=\bm{1\cdot 10^{-17}}$& $p_{\text{Wil}}=\bm{1.9\cdot 10^{-18}}$\\
\cline{2-4}
 & \multirow{2}{*}{Silverman}  & $0.49$ ($0.45$, $0.53$) & $0.00045$ ($0.00029$, $0.00073$)\\
 & & $p_{\text{Wil}}=\bm{1.9\cdot 10^{-18}}$& $p_{\text{Wil}}=\bm{5.5\cdot 10^{-16}}$\\
\hline
 \multirow{7}{*}{\makecell{Appliances\\Energy Use}} & Jacobian  & $0.26$ ($0.24$, $0.28$) & $0.00036$ ($0.00026$, $0.00042$)\\
\cline{2-4}
 & \multirow{2}{*}{GCV}  & $0.27$ ($0.24$, $0.3$) & $200$ ($180$, $220$)\\
 & & $p_{\text{Wil}}=1$& $p_{\text{Wil}}=\bm{1.9\cdot 10^{-18}}$\\
\cline{2-4}
 & \multirow{2}{*}{MML}  & $0.27$ ($0.24$, $0.29$) & $240$ ($210$, $270$)\\
 & & $p_{\text{Wil}}=1$& $p_{\text{Wil}}=\bm{1.9\cdot 10^{-18}}$\\
\cline{2-4}
 & \multirow{2}{*}{Silverman}  & $0.18$ ($0.16$, $0.21$) & $0.0011$ ($0.00086$, $0.0017$)\\
 & & $p_{\text{Wil}}=\bm{1.9\cdot 10^{-18}}$& $p_{\text{Wil}}=\bm{2.9\cdot 10^{-17}}$\\
\hline
 \multirow{7}{*}{\makecell{California\\Housing}} & Jacobian  & $0.59$ ($0.29$, $0.71$) & $0.00034$ ($0.00017$, $0.00044$)\\
\cline{2-4}
 & \multirow{2}{*}{GCV}  & $0.7$ ($0.59$, $0.75$) & $590$ ($340$, $790$)\\
 & & $p_{\text{Wil}}=1$& $p_{\text{Wil}}=\bm{1.9\cdot 10^{-18}}$\\
\cline{2-4}
 & \multirow{2}{*}{MML}  & $0.38$ ($-0.045$, $0.65$) & $260$ ($240$, $280$)\\
 & & $p_{\text{Wil}}=\bm{8.8\cdot 10^{-8}}$& $p_{\text{Wil}}=\bm{1.9\cdot 10^{-18}}$\\
\cline{2-4}
 & \multirow{2}{*}{Silverman}  & $0.5$ ($0.45$, $0.55$) & $0.00054$ ($0.00039$, $0.00083$)\\
 & & $p_{\text{Wil}}=\bm{2.9\cdot 10^{-7}}$& $p_{\text{Wil}}=\bm{1.9\cdot 10^{-18}}$\\
\hline
 \multirow{7}{*}{\makecell{Protein\\Structure}} & Jacobian  & $0.37$ ($0.35$, $0.39$) & $0.00033$ ($0.00018$, $0.00042$)\\
\cline{2-4}
 & \multirow{2}{*}{GCV}  & $0.42$ ($0.39$, $0.44$) & $460$ ($400$, $510$)\\
 & & $p_{\text{Wil}}=1$& $p_{\text{Wil}}=\bm{1.9\cdot 10^{-18}}$\\
\cline{2-4}
 & \multirow{2}{*}{MML}  & $0.32$ ($0.3$, $0.33$) & $300$ ($240$, $370$)\\
 & & $p_{\text{Wil}}=\bm{1.9\cdot 10^{-18}}$& $p_{\text{Wil}}=\bm{1.9\cdot 10^{-18}}$\\
\cline{2-4}
 & \multirow{2}{*}{Silverman}  & $0.14$ ($0.086$, $0.19$) & $0.00066$ ($0.00043$, $0.0011$)\\
 & & $p_{\text{Wil}}=\bm{1.9\cdot 10^{-18}}$& $p_{\text{Wil}}=\bm{1.9\cdot 10^{-18}}$\\
\hline
 \multirow{7}{*}{\makecell{U.K.\ Black\\Smoke}} & Jacobian  & $0.31$ ($0.29$, $0.33$) & $0.00032$ ($0.00016$, $0.00041$)\\
\cline{2-4}
 & \multirow{2}{*}{GCV}  & $0.31$ ($0.3$, $0.33$) & $380$ ($330$, $420$)\\
 & & $p_{\text{Wil}}=1$& $p_{\text{Wil}}=\bm{1.9\cdot 10^{-18}}$\\
\cline{2-4}
 & \multirow{2}{*}{MML}  & $0.27$ ($0.26$, $0.29$) & $240$ ($220$, $270$)\\
 & & $p_{\text{Wil}}=\bm{1.9\cdot 10^{-18}}$& $p_{\text{Wil}}=\bm{1.9\cdot 10^{-18}}$\\
\cline{2-4}
 & \multirow{2}{*}{Silverman}  & $0.16$ ($0.12$, $0.19$) & $0.00056$ ($0.00041$, $0.0011$)\\
 & & $p_{\text{Wil}}=\bm{1.9\cdot 10^{-18}}$& $p_{\text{Wil}}=\bm{2.7\cdot 10^{-17}}$\\
\hline
\end{tabular}
\label{tab:big_res}
\end{table}

\subsection{Synthetic and Simpler Real Data}
\label{sec:simple_exps}
In this section, we compare the four methods on the following data:
\begin{itemize}
\item 2D Temperature Data: The temperatures at 40 different French weather stations at 3 a.m., Jan 1st, 2020.
\item 1D Temperature Data: The temperature at 248 different times in January 2020 at the weather station at Toulouse-Blagnac.
\item Cauchy Synthetic Data: For $n$ observations, for $i\in [1,n]$, $x_i\sim\text{Cauchy}(0,3)$, $y_i=\sin(2\pi x_i)+\varepsilon_i$, where $\varepsilon_i\sim\mathcal{N}(0,0.2^2)$.
\end{itemize}
The French temperature data was obtained from Météo France\footnote{The data set is available at \url{https://donneespubliques.meteofrance.fr/donnees_libres/Txt/Synop/postesSynop.csv}} and processed following the setup by \citet{vanwynsberge1999kriging}. For the cross-validation, 100 logarithmically spaced values between $0.001$ and $l_{\max}$ were used.

In Figure \ref{fig:bw_sim_demo}, we compare the approximate and true Jacobian norms as functions of the bandwidth. We see that the approximate norm captures the structure of the true norm quite well. In the absence of regularization, the minima of the two functions approximately agree. When regularization is added, the selected bandwidth, $\sigma_0$, is close to the elbow of both the approximate and true norms. 

\begin{figure}
\center
\includegraphics[width=\textwidth]{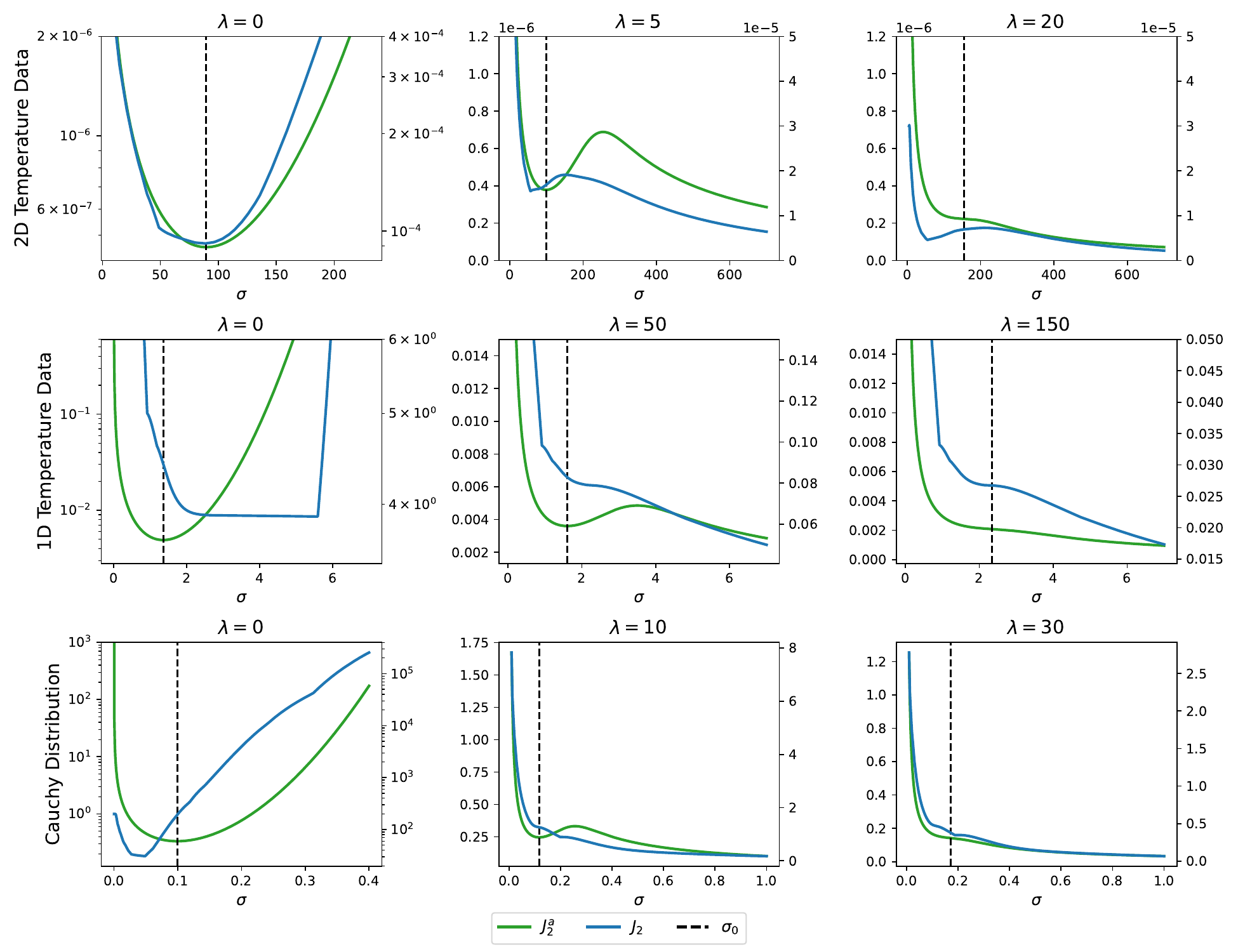}
\caption{Comparison of the approximate (green, left y-axis) and true (blue, right y-axis) Jacobian norms as a function of bandwidth and regularization. The approximate Jacobian norm captures the structure of the true Jacobian norm quite well, especially for the 2D temperature data, where for $\lambda=0$, the minima of the two functions agree very well. For $\lambda>0$, the selected bandwidth, $\sigma_0$, is close to the elbow of both the approximate and true norms. In the rightmost panel, $\lambda>2 n e^{-3/2}$, which means that the approximate Jacobian norm has no local minimum and $\sigma_0$ is selected as if $\lambda=2 n e^{-3/2}$.}
\label{fig:bw_sim_demo}
\end{figure}

In Figures \ref{fig:french_2d} and \ref{fig:french_1d}, we provide the results of jackknife resampling for the two temperature data sets. For each data set, the experiments were repeated $n$ times, omitting a different observation in each experiment. For the 1D temperature data, half of the observations were set aside as reference data. Apart from plotting the jackknife mean and standard deviations of the prediction, we also state the mean and standard deviations of the selected bandwidths. Compared to the other methods, the Jacobian method tends to provide less extreme predictions and to be more stable in terms of selected bandwidth across the experiments. For the 1D data, the predictions provided by MML and Silverman’s method are admittedly the least extreme, but at the expense of failing to capture the dynamics of the data.

\begin{figure}
\center
\includegraphics[width=\textwidth]{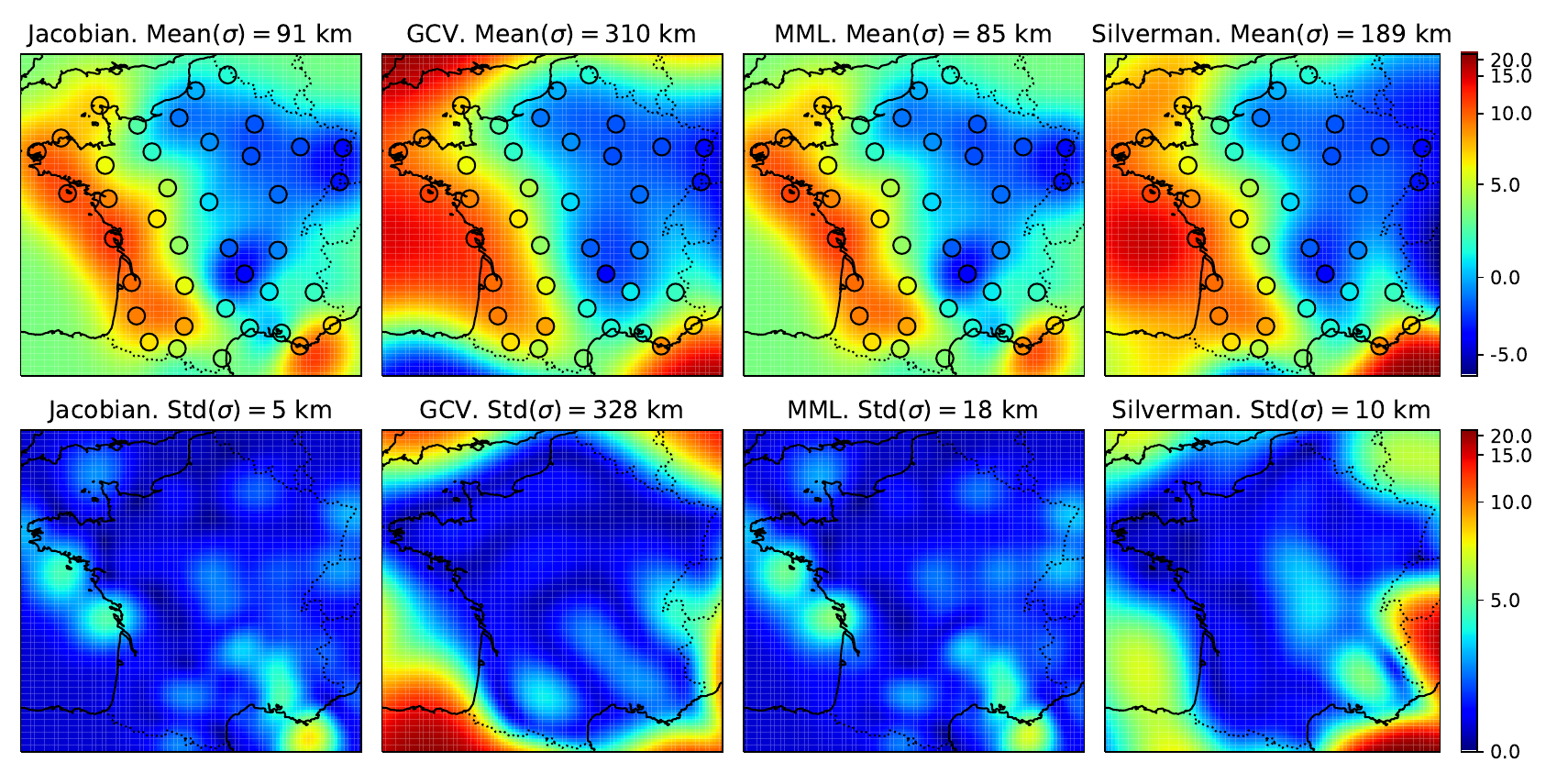}
\caption{Means (top row) and standard deviations (bottom row) of KRR temperature predictions in $^\circ C$ from jackknife resampling on the 2D temperature data. Note that the scales are not linear. 
The Jacobian and MML methods provide less extreme predictions than GCV and Sliverman's method does. They are also more stable in terms of bandwidth selection.}
\label{fig:french_2d}
\end{figure}

\begin{figure}
\center
\includegraphics[width=\textwidth]{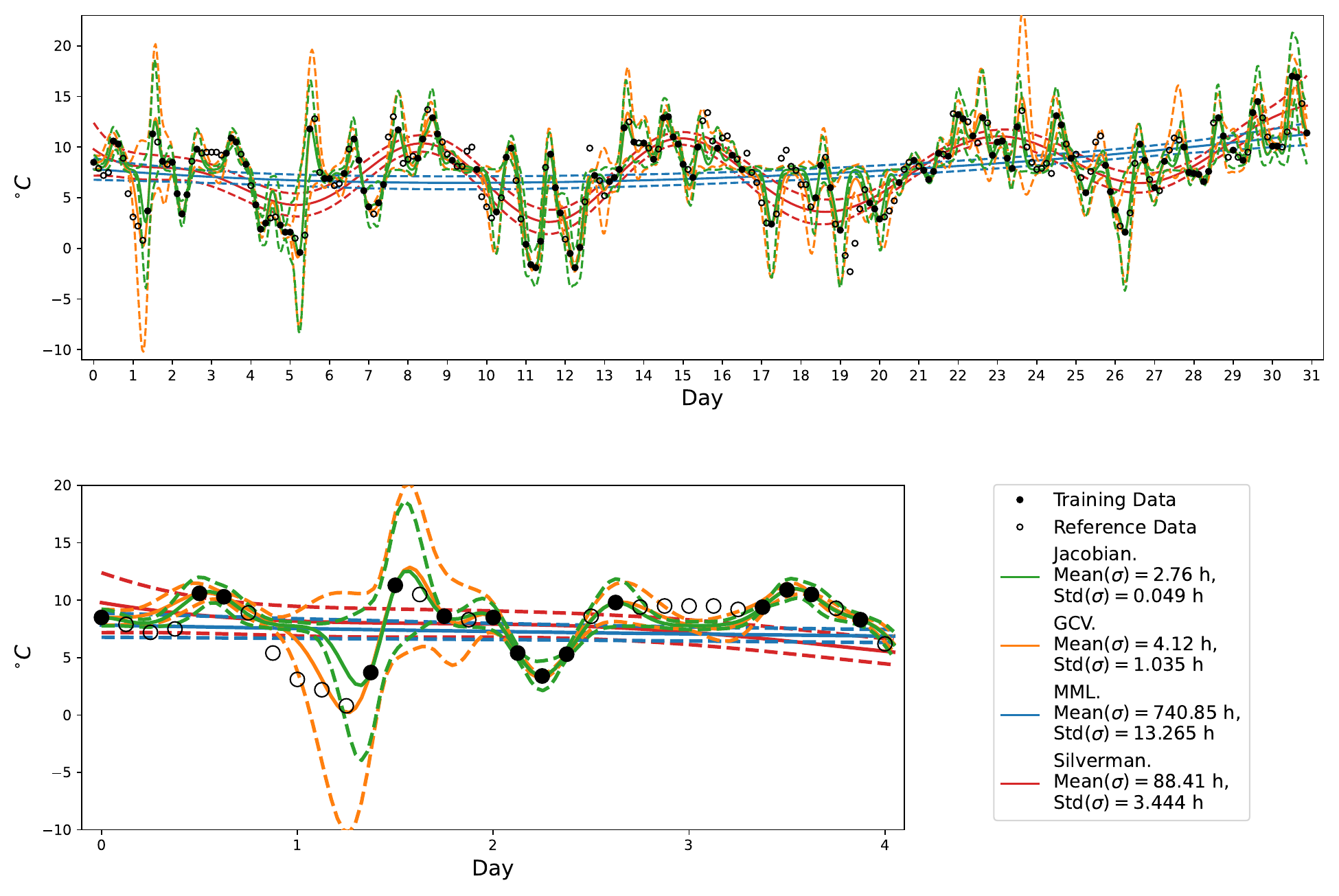}
\caption{Means and standard deviations of KRR predictions from jackknife resampling on the 1D temperature data. The lower bottom plot shows a zoom-in on the first 4 days. The Jacobian method performs similarly to cross-validation but provides slightly less extreme predictions. It is also more stable in terms of bandwidth selection. MML and Silverman's methods both underfit the data, which can be attributed to their much larger bandwidth.}
\label{fig:french_1d}
\end{figure}

\subsubsection{Varying Sample Size and Regularization}
In Figures \ref{fig:sweep_n} and \ref{fig:sweep_lbda}, we vary the sample size, $n$, and regularization strength, $\lambda$.
For the synthetic data, 1000 test observations were generated, while the real data was randomly split into training and testing data. When varying $n$, 15\ \% of the real data was saved for testing, resulting in splits of size  $n$/6, $n$/37, and $n$/1000 for the 2D temperature, 1D temperature, and synthetic data, respectively.
When varying $\lambda$, $n$ was chosen to a value where the different methods performed approximately equally well in the experiments with varying sample size (Figure \ref{fig:sweep_n}). Thus the splits when varying $\lambda$ were 25/15, 100/148, and 50/1000 for the 2D temperature, 1D temperature, and synthetic data, respectively.
In all cases 1000 random splits were used to estimate the variance of test $R^2$, the selected bandwidth $\sigma$, and the bandwidth selection computation time, $t$. 
It is again confirmed that the Jacobian method, in addition to being much faster than GCV and MML, is much more stable in terms of bandwidth selection. For the Cauchy distributed data, the median version of the Jacobian method was used; this method requires slightly more time than the standard Jacobian method. The reason for Silverman's method being slower than the Jacobian method for the 2D temperature data is due to its need to calculate the standard deviation of the data, and thus the distance to the mean from all observations. For the 2D temperature data, calculating the distances comprises a larger fraction of the calculations than for the other data. The other three methods do not use the standard deviation and are thus less affected.

\begin{figure}
\centering
\includegraphics[width=\textwidth]{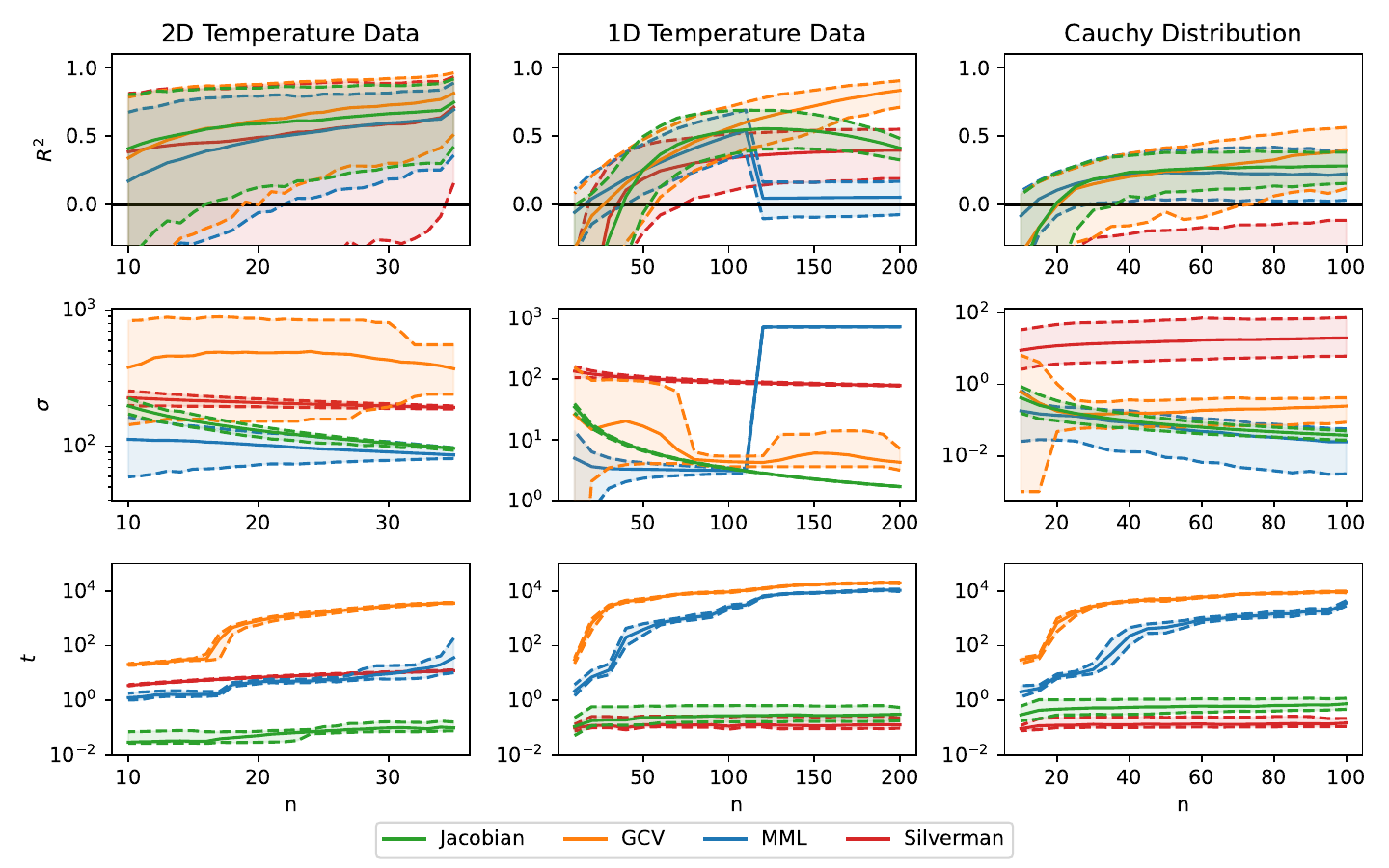}
\caption{Mean together with first and ninth deciles for explained variance on test data, $R^2$; selected bandwidth, $\sigma$; and computation time in milliseconds $t$, for different training sample sizes, using the four bandwidth selection methods. 
The Jacobian and Silverman's methods are several orders of magnitude faster than the two other methods. They are also much more stable in terms of bandwidth selection. In terms of prediction, the Jacobian method generally performs better than, or on par with, the competing methods, except compared to cross-validation when $n$ is large. For the 1D temperature data, MML gets stuck in a local minimum. For the Cauchy data, the, slightly slower, median version of the Jacobian method was used.}
\label{fig:sweep_n}
\end{figure}

\begin{figure}
\centering
\includegraphics[width=\textwidth]{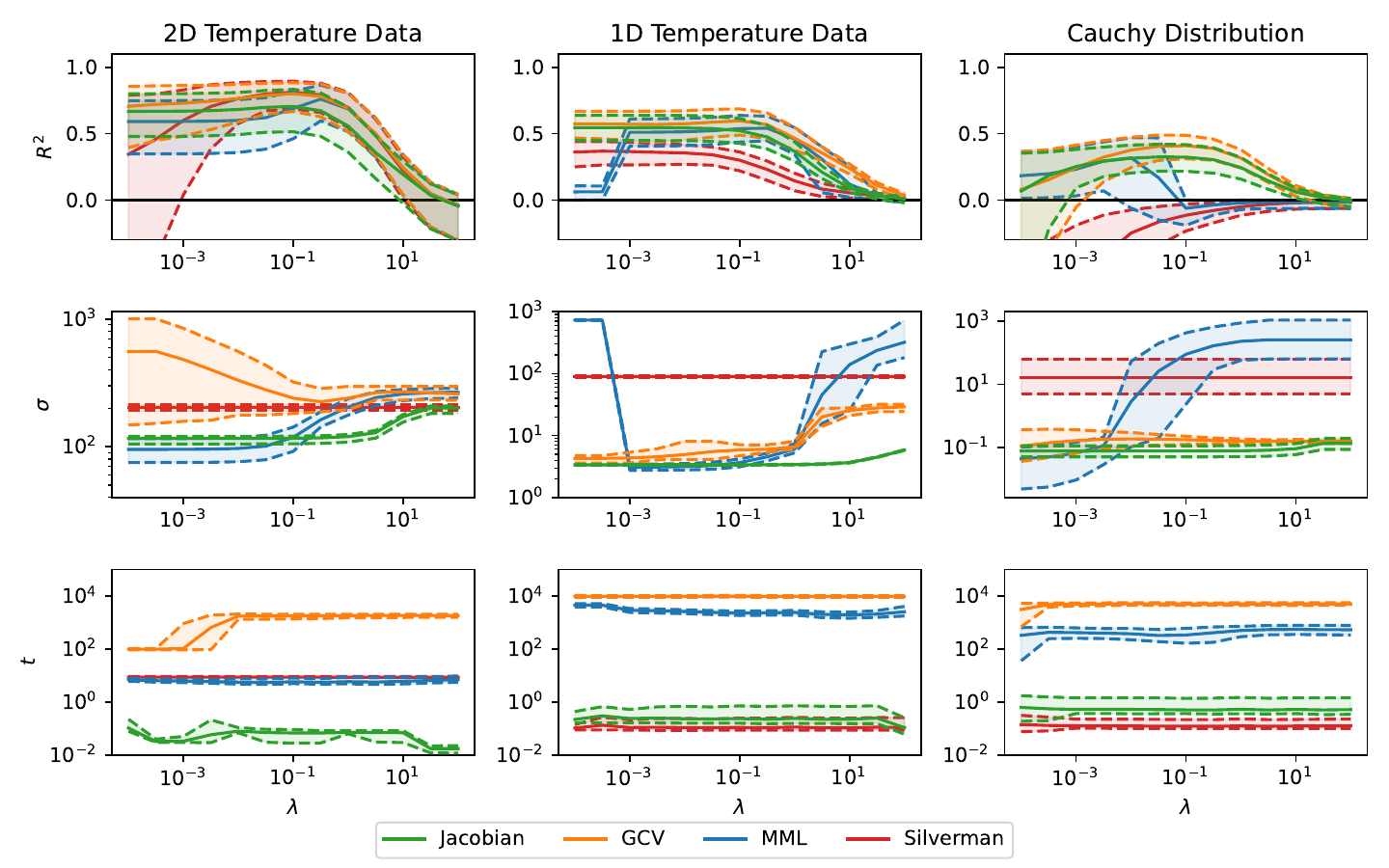}
\caption{Mean together with first and ninth deciles for explained variance on test data, $R^2$; selected bandwidth, $\sigma$; and computation time in milliseconds $t$, for different regularization strengths, using the four bandwidth selection methods. 
The Jacobian and Silverman's methods are several orders of magnitude faster than the two other methods. They are also much more stable in terms of bandwidth selection. In terms of prediction, the Jacobian method generally performs better than, or on par with, the competing methods. For both the 1D temperature data and the Cauchy data, MML gets stuck in local minima. For the Cauchy data, the, slightly slower, median version of the Jacobian method was used.}
\label{fig:sweep_lbda}
\end{figure}

\section{Conclusions}
We proposed a computationally low-cost method for choosing the bandwidth for kernel ridge regression with the Gaussian kernel. The method was motivated by the observed improved generalization properties of neural networks trained with Jacobian regularization. 
By selecting a bandwidth with Jacobian control, we implicitly constrain the derivatives of the inferred function.  To achieve this, we derived a simple, closed-form approximate expression for the Jacobian of Gaussian KRR as a function of the bandwidth and were thus able to find an optimal bandwidth for Jacobian control.
We demonstrated how selecting the optimal bandwidth is a trade-off between utilizing a well-conditioned training data kernel matrix and a slow decay of the inferred function toward the default value. 

In terms of model performance, our method is on par with cross-validation and marginal likelihood maximization, but up to a million times faster on the considered data. Compared to Silverman's method, our method is superior in terms of model performance.

Even though we only considered Jacobian bandwidth selection for the Gaussian kernel, the principle holds for any kernel. That, however, requires new, kernel-specific, estimates of the kernel matrix norms. Similarly, in the Gaussian case, the estimate of the norm of the inverse kernel matrix could potentially be further improved. These research problems are left for future work.

Code is available at \url{https://github.com/allerbo/jacobian_bandwidth_selection}.


\clearpage
\appendix

\section{Proofs}
\label{sec:proofs}

\begin{proof}[Proof of Proposition \ref{thm:best_sigma}]~\\
Denote $d:=\frac{2l_{\max}}{((n-1)^{1/p}-1)\pi}$. Then
$$J^a_2(\sigma,l,n,p,\lambda)= J^a_2(\sigma,n,d,\lambda)=\frac1{\sigma\left(n\exp\left(-\left(\frac\sigma d\right)^2\right)+\lambda\right)},$$
from which we obtain
\begin{align*}
\lim_{\sigma \to 0^+}&J^a_2(\sigma)=+\infty
\intertext{and}
\lim_{\sigma \to +\infty}&J^a_2(\sigma)=
\begin{cases}
+\infty\text{ if }\lambda=0\\
0\text{ if }\lambda>0.\\
\end{cases}\\
\end{align*}

We now identify stationary points by setting the derivative to 0.
$$\frac{\partial J^a_2(\sigma)}{\partial\sigma}=-\frac{\exp\left(\left(\frac\sigma d\right)^2\right)\left(n+\lambda\exp\left(\left(\frac\sigma d\right)^2\right)-2n\left(\frac\sigma d\right)^2\right)}{\left(\sigma\left(n+\lambda\exp\left(\left(\frac\sigma d\right)^2\right)\right)\right)^2}.$$
\begin{equation*}
\begin{aligned}
&n+\lambda e^{\left(\frac\sigma d\right)^2}-2n\left(\frac\sigma d\right)^2=0\iff-\frac{\lambda \sqrt{e}} {2n}=e^{\frac12-\left(\frac\sigma d\right)^2}\left(\frac12-\left(\frac\sigma d\right)^2\right)\\
&\iff\left(\frac12-\left(\frac\sigma d\right)^2\right)=W\left(-\frac{\lambda \sqrt{e}} {2n}\right)\implies\sigma=\frac d{\sqrt{2}} \sqrt{1-2W\left(-\frac{\lambda \sqrt{e}} {2n}\right)},
\end{aligned}
\end{equation*}
where $W$ denotes the Lambert $W$ function. Since this function has real outputs only if its argument is greater than $-e^{-1}$, in order to obtain stationary points we need
$$-\frac{\lambda \sqrt{e}} {2n}\geq -e^{-1}\iff \lambda\leq 2ne^{-3/2}$$
which gives us the two stationary points 
\begin{align*}
&\sigma_0=\frac{\sqrt{2}}{\pi}\frac{l_{\max}}{(n-1)^{1/p}-1} \sqrt{1-2W_0\left(-\frac{\lambda \sqrt{e}} {2n}\right)}
\intertext{and}
&\sigma_{-1}=\frac{\sqrt{2}}{\pi}\frac{l_{\max}}{(n-1)^{1/p}-1} \sqrt{1-2W_{-1}\left(-\frac{\lambda \sqrt{e}} {2n}\right)}.
\end{align*}
$W_{-1}(x)<W_0(x)$ for $x\in(-e^{-1},0)$, which implies that $\sigma_0<\sigma_{-1}$. Combined with the limits above, this implies that, when existing, $\sigma_0$ is a local minimum and $\sigma_{-1}$ is a local maximum.

Finally, for $\lambda=0$, $W_{0}(0)=0$ and $\lim_{\lambda\to0}W_{-1}\left(-\frac{\lambda\sqrt{e}}{2n}\right)=-\infty$, which means that in the absence of $\lambda$, $\sigma_0=\frac{\sqrt{2}}{\pi}\frac{l_{\max}}{(n-1)^{1/p}-1}$ and $\sigma_{-1}=+\infty$.
\end{proof}

\begin{proof}[Proof of Proposition \ref{thm:dydx}]~\\
We first note that for $\dvi=\xsv-\xvi$, $\frac{\partial \fh(\xsv)}{\partial\dvi}=\frac{\partial \fh(\xsv)}{\partial\xsv}$:
\begin{equation*}
\begin{aligned}
&\frac{\partial \dvi}{\partial\xsv}=\frac{\partial(\xsv-\xvi)}{\partial \xsv}=\frac{\partial \xsv}{\partial \xsv}-\frac{\partial \xvi}{\partial \xsv}=\I-\bm{0}=\I\\
&\frac{\partial \fh(\xsv)}{\partial \xsv}=\frac{\partial \fh(\xsv)}{\partial \dvi}\cdot\frac{\partial \dvi}{\partial \xsv}=\frac{\partial \fh(\xsv)}{\partial \dvi}\cdot\I=\frac{\partial \fh(\xsv)}{\partial \dvi}.
\end{aligned}
\end{equation*}
Now,
\begin{equation*}
\begin{aligned}
\left\|\frac{\partial \fh(\xsv)}{\partial \dvi}\right\|_{2}&=\left\|\frac{\partial \fh(\xsv)}{\partial \xsv}\right\|_{2}
=\left\|\frac{\partial}{\partial\xsv}\left(\kv(\xsv,\X)^\top\cdot \left(\K(\X,\X)+\lambda\I\right)^{-1}\cdot\yv\right)\right\|_2\\
=&\left\|\frac{\partial \kv(\xsv,\X)^\top}{\partial \xsv}\cdot\left(\K(\X,\X)+\lambda\I\right)^{-1}\yv\right\|_2\\
\leq&\left\|\frac{\partial \kv(\xsv,\X)^\top}{\partial \xsv}\right\|_2\cdot\left\|\left(\K(\X,\X)+\lambda\I\right)^{-1}\right\|_2\cdot\left\|\yv\right\|_2\\
\leq&\sqrt{n}\cdot\left\|\frac{\partial \kv(\xsv,\X)^\top}{\partial \xsv}\right\|_1\cdot\left\|\left(\K(\X,\X)+\lambda\I\right)^{-1}\right\|_2\cdot\left\|\yv\right\|_2\\
=&\sqrt{n}\cdot\left\|\begin{bmatrix}\frac{\partial k(\xsv,\bm{x_1})}{\partial \xsv} & \frac{\partial k(\xsv,\bm{x_2})}{\partial \xsv}  & \dots & \frac{\partial k(\xsv,\bm{x_n})}{\partial \xsv} \end{bmatrix}\right\|_1\\
&\cdot\left\|\left(\K(\X,\X)+\lambda\I\right)^{-1}\right\|_2\cdot\left\|\yv\right\|_2\\
=&\sqrt{n}\cdot\max_{\xvi\in\X}\left\|\frac{\partial k(\xsv,\xvi)}{\partial \xsv}\right\|_1\cdot\left\|\left(\K(\X,\X)+\lambda\I\right)^{-1}\right\|_2\cdot\left\|\yv\right\|_2\\
\stackrel{(a)}=&\sqrt{n}\cdot\max_{\xvi\in\X}\left\|\frac{\partial k(\xvi+\dvi,\xvi)}{\partial \dvi}\right\|_1\cdot\left\|\left(\K(\X,\X)+\lambda\I\right)^{-1}\right\|_2\cdot\left\|\yv\right\|_2,\\
\end{aligned}
\end{equation*}
where in $(a)$, we used the chain rule together with $\frac{\partial \dvi}{\partial \xsv}=\I$.
\end{proof}

\begin{proof}[Proof of Proposition \ref{thm:max_der_gauss}]~\\
In spherical coordinates,
$$\left\|\frac{\partial k_G(\dvi,\sigma)}{\partial \dvi}\right\|_1=\left|\frac{\partial k_G(\dvi,\sigma)}{\partial d_i}\right|+\sum_{j=2}^p\left|\frac1{d_i}\frac{\partial k_G(\dvi,\sigma)}{\partial \theta_j}\right|,$$
where the sum is over the angular coordinates. Since the Gaussian kernel is rotationally invariant, this sum is 0 and 
$$\left\|\frac{\partial k_G(\dvi,\sigma)}{\partial \dvi}\right\|_1=\left|\frac{\partial}{\partial d_i}\exp\left(-\frac{d_i^2}{2\sigma^2}\right)\right|=\frac{d_i}{\sigma^2}\exp\left(-\frac{d_i^2}{2\sigma^2}\right).$$
To find the $d_i$ that maximizes the derivative, we look where the second derivative is zero.

$$\frac{\partial}{\partial d_i}\left|\frac{\partial k_G(d_i,\sigma)}{\partial d_i}\right|=\left(\left(\frac{d_i}{\sigma^2}\right)^2-\frac1{\sigma^2}\right)\exp\left(-\frac{d_i^2}{2\sigma^2}\right).$$

Setting the second derivative to zero amounts to
$$\left(\frac{d_i}{\sigma^2}\right)^2=\frac1{\sigma^2}\iff d_i^2=\sigma^2\implies d_i=\sigma.$$
Plugging this into the first derivative we obtain $\frac1\sigma\exp\left(-\frac12\right)$, which is greater than 
$$\left|\frac{\partial k_G(0,\sigma)}{\partial d_i}\right|=\left|\frac{\partial k_G(\infty,\sigma)}{\partial d_i}\right|=0,$$
and consequently
$$\max_{\dvi}\left\|\frac{\partial k_G(\dvi,\sigma)}{\partial \dv}\right\|_1=\max_{d_i}\left|\frac{\partial k_G(d_i,\sigma)}{\partial d_i}\right|=\frac1{\sigma\sqrt{e}}.$$
\end{proof}

\begin{proof}[Proof of Proposition \ref{thm:min_ki_gauss}]~\\
To alleviate notation, from now on we do not explicitly state that $\K$ depends on $\X$. We first note that $\left\|\left(\K+\lambda\I\right)^{-1}\right\|_2=\frac1{s_n\left(\K+\lambda\I\right)}$, where $s_n$ denotes the smallest singular value of $\K$. Since $\K$ is symmetric and positive semi-definite, it is diagonalizable as $\K=\U\bm{\Sigma}\U^\top$, while $\lambda\I=\lambda\U\U^\top$, which means that $\K+\lambda\I=\U\left(\bm{\Sigma}+\lambda\I\right)\U^\top$, i.e.\ the singular values of $\K+\lambda\I$ are the singular values of $\K$, shifted by $\lambda$.

According to \citet{bermanis2013multiscale}, 
for $\xv\in \R^p$, where each $x_i$ is restricted to an interval of length $l_i$, $i=1,2,\dots p$, for a Gaussian kernel matrix $\K\in \R^{m\times n}$, with singular values $s_1,\dots s_n$, the number of singular values larger than $\delta\cdot s_1$ for some $\delta>0$, $R_\delta(\K)$, is bounded according to
\begin{equation*}
\begin{aligned}
R_{\delta}(\K)&:=\#\left\{j:\frac{s_j(\K)}{s_1(\K)}\geq \delta\right\}\leq \prod_{i=1}^d\left(\frac 2\pi \frac {l_i}\sigma\sqrt{\log(1/\delta)}+1\right)\\
&\leq\left(\frac 2\pi \frac {l_{\max}}\sigma\sqrt{\log(1/\delta)}+1\right)^p.
\end{aligned}
\end{equation*}
Solving for $\delta$, we obtain
\begin{equation*}
\begin{aligned}
&R_\delta(\K)\leq \left(\frac 2\pi \frac {l_{\max}}\sigma\sqrt{\log(1/\delta)}+1\right)^p\\
\iff &(R_\delta(\K)^{1/p}-1)\frac {\pi\sigma}{2l_{\max}}\leq\sqrt{\log(1/\delta)}\\
\iff &\delta\leq \exp\left(-\left(\frac{(R_\delta(\K)^{1/p}-1)\pi\sigma}{2l_{\max}}\right)^2\right).
\end{aligned}
\end{equation*}
Now, if $R_\delta(\K)=n$, then all singular values (including $s_n$) are larger than or equal to $\delta\cdot s_1$. If $R_\delta(\K)=n-1$, then all but one (namely $s_n$) of the singular values are larger than or equal to $\delta\cdot s_1$. So for $R_\delta(\K)=n-1$, $s_n< \delta \cdot s_1$, which implies

\begin{equation*}
\begin{aligned}
s_n&<s_1\delta\leq s_1\exp\left(-\left(\frac{((n-1)^{1/p}-1)\pi\sigma}{2l_{\max}}\right)^2\right)\\
&\leq n\exp\left(-\left(\frac{((n-1)^{1/p}-1)\pi\sigma}{2l_{\max}}\right)^2\right),
\end{aligned}
\end{equation*}
where we used $s_1(\K)\cdot\leq n\cdot\|\K\|_{\max}=n\cdot1$.

Thus 
$$\left\|\left(\K+\lambda\I\right)^{-1}\right\|_2=\frac1{s_n+\lambda} \geq \frac1{n\exp\left(-\left(\frac{((n-1)^{1/p}-1)\pi\sigma}{2l_{\max}}\right)^2\right)+\lambda}.$$
\end{proof}

\clearpage
\bibliography{refs}
\bibliographystyle{apalike}
\end{document}